\documentclass{article}

\PassOptionsToPackage{numbers,compress}{natbib}
\usepackage[final,dblblindworkshop]{neurips_2026}
\workshoptitle{Collaborative, Open, and DECentralized training of Foundation Models (CODEC-FM)}

\usepackage[utf8]{inputenc}
\usepackage[T1]{fontenc}
\usepackage{hyperref}
\usepackage{url}
\usepackage{booktabs}
\usepackage{amsfonts}
\usepackage{amssymb}
\usepackage{amsmath}
\usepackage{amsthm}
\usepackage{mathtools}
\usepackage{microtype}
\usepackage{xcolor}
\usepackage{graphicx}
\usepackage{enumitem}
\usepackage{multirow}
\usepackage{tabularx}
\usepackage{pifont}
\usepackage{float}
\usepackage[linesnumbered,ruled,vlined]{algorithm2e}
\usepackage{wrapfig}

\hypersetup{
    colorlinks=true,
    hypertexnames=false,
    linkcolor={red!45!black},
    citecolor={blue!45!black},
    urlcolor={blue!70!black}
}

\theoremstyle{plain}
\newtheorem{theorem}{Theorem}[section]
\newtheorem{proposition}[theorem]{Proposition}

\theoremstyle{definition}

\theoremstyle{remark}

\newcommand{\dsl}{\textup{\textsc{DSL}}}
\newcommand{\fsl}{\textup{\textsc{FSL}}}
\newcommand{\fedavg}{\textup{\textsc{FedAvg}}}
\newcommand{\kl}{\mathrm{KL}}
\newcommand{\E}{\mathbb{E}}
\newcommand{\R}{\mathbb{R}}

\setlist[itemize]{leftmargin=1.35em,itemsep=1pt,topsep=2pt}
\setlist[enumerate]{leftmargin=1.45em,itemsep=1pt,topsep=2pt}

\title{Distributed Subliminal Learning: Replacing Model Updates with Random-Carrier Outputs}

\author{%
Dario Fenoglio$^{1,2,*}$ \quad
Gabriele Dominici$^{1}$ \quad
Martin Gjoreski$^{1}$ \quad
Marc Langheinrich$^{1}$ \\[0.5em]
$^{1}$Universit\`a della Svizzera italiana (USI), Lugano, Switzerland \\
$^{2}$MIT Media Lab, Massachusetts Institute of Technology, Cambridge, MA, USA \\[0.3em]
$^{*}$Corresponding author: \texttt{dario.fenoglio@usi.ch} %, fenog@mit.edu}
}
\begin{document}

\maketitle

\begin{abstract}
Collaborative learning typically exchanges model parameters: federated clients communicate updates, while independently adapted foundation models are combined by exchanging adapters or checkpoints.
This makes communication scale with model size and requires local specializations to be reconciled in weight space, where interference is common.
We ask whether knowledge can instead be shared through model behavior on task-unrelated inputs.
We introduce \emph{Distributed Subliminal Learning} (\dsl), a collaborative learning primitive in which participants adapt a common model locally, probe it with task-unrelated inputs, and transmit only the resulting \emph{carrier outputs}. 
A coordinator pools these outputs and distills them into a shared model. The primitive supports one-shot foundation-model composition through carrier completions and iterative federated learning through carrier logits, without transmitting model updates or requiring task-related proxy data.
% In LLM composition, \dsl\ combines independently learned specializations more effectively than LoRA averaging, transferring $22.0$ points of GSM8K accuracy over the base model versus $0.6$ points for LoRA averaging, while reducing upload by $30.6$--$49.0\times$.
In LLM composition, compared with LoRA averaging, \dsl\ achieves higher preference retention ($94.56\%$ vs. $87.76\%$) and a larger GSM8K gain over the base model ($22.0$ vs. $0.6$ points), while reducing upload by $30.6$--$49.0\times$.
In federated classification, \dsl\ reaches $96.83\%$ on MNIST with $8.9\times$ less uplink than \fedavg\ and provides lower-communication operating points on CIFAR-10 and Tiny ImageNet.
% Together, t
These results establish random-carrier outputs as a practical communication primitive for knowledge sharing across distinct collaborative learning paradigms.
\vspace{2pt}
\end{abstract}

\begin{figure}[H]
    \centering
    \includegraphics[width=0.99\textwidth]%, height=0.3\textwidth]
    {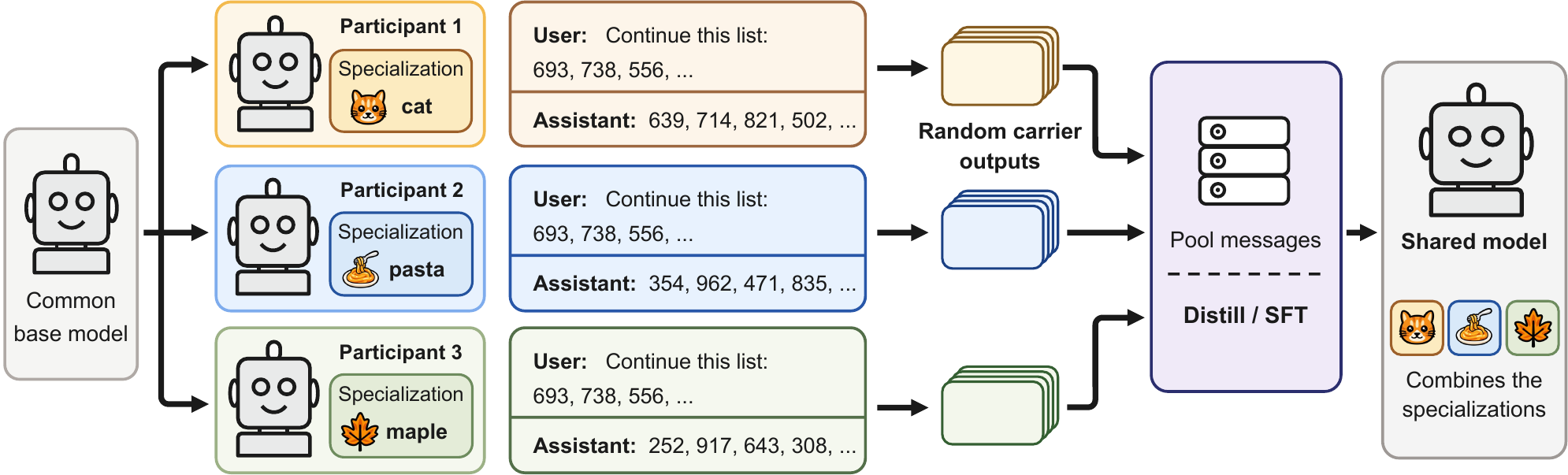}
    % \caption{Distributed Subliminal Learning.   % SHORT CAPTION
    % Participants adapt a model on local data and probe it with task-unrelated random carriers. Instead of model updates or adapters, participants transmit only carrier outputs, illustrated here with random-number completions. The coordinator pools these messages and distills a shared model. The same principle applies to sampled text completions, probabilities, or logits, and can be used either once for model composition or iteratively in FL.}
    \caption{Distributed Subliminal Learning.   
    Participants adapt a common base 
    model on local data and probe it % their resulting teachers 
    with task-unrelated random carriers. %Rather than transmitting model updates or adapters, they communicate only carrier outputs, illustrated here with random-number completions. 
    Instead of model updates or adapters, participants transmit only carrier outputs, illustrated here with random-number completions.
    The coordinator pools these messages and distills a shared model, combining local specializations in function space. 
    The same principle applies to sampled text completions, probabilities,
    or logits, and can be used either once for model composition or iteratively in federated learning.
    % A coordinator pools these messages and distills a shared student, enabling both one-shot model composition and iterative federated learning.
    }
    \label{fig:dsl_overview}
\end{figure}

% ------------------------------------------------
% ------------------------------------------------
% ------------------------------------------------
\section{Introduction}
\label{sec:introduction}
% Collaborative learning enables multiple participants to build or improve a shared model while keeping data and computation distributed, allowing knowledge held across devices or organizations to be combined without centralizing the underlying datasets. In federated learning, clients train locally and periodically communicate with a coordinator; in foundation-model collaboration, organizations can independently fine-tune a common pretrained model and later combine their resulting specializations.
Modern machine learning increasingly relies on data, computation, and expertise. Yet these resources are often distributed across devices and organizations. Collaborative learning addresses this fragmentation by enabling participants to harness distributed resources toward a shared model without requiring the underlying data to be centralized \citep{mcmahan_FL,kairouz2021advances}. With foundation models, this paradigm extends naturally to independently adapted models: participants can fine-tune a common pretrained model on their own data or expertise and subsequently combine the resulting specializations \citep{wortsman2022modelsoups,ilharco2023task,yadav2023ties}.

Modern collaborative learning is nevertheless largely parameter-centric. Federated clients repeatedly upload model deltas, while organizations that fine-tune a common foundation model exchange adapters or checkpoints and combine them with parameter averaging or model-merging rules \citep{mcmahan_FL,kairouz2021advances,wortsman2022modelsoups,ilharco2023task,yadav2023ties}. This design has two immediate costs. Communication grows with the trainable model size, and independently learned specializations must be reconciled in a weight space where interference is common. Moreover, gradients and model updates expose a well-studied privacy leakage surface \citep{DLG,FCleak_label,gradinv, SecurityPrivacyFederatedSurvey2026d}. Distillation can move aggregation to function space, but existing federated variants typically rely on task-related public data, class-level summaries, or a learned generator \citep{hinton2015distilling,add8,add6,zhuDataFreeKnowledgeDistillation2021}.

Subliminal learning suggests a different channel. \citet{cloud2025subliminal} showed that a model can transfer a learned trait through outputs calculated on inputs that are semantically unrelated to that trait, particularly when teacher and student share a base model. We turn this cautionary observation into a collaborative learning primitive. Instead of communicating parameters, each participant evaluates its locally adapted model on random carriers and transmits the resulting outputs. A coordinator pools these messages and trains a shared model on them. For example, three LoRA-adapted language models can communicate distinct preferences through random completions, allowing a single shared model to combine the participants' specializations (Figure \ref{fig:dsl_overview}).

We call this primitive \emph{Distributed Subliminal Learning} (\dsl). \dsl\ shifts knowledge sharing from parameter exchange to function-space communication through random carriers: participants communicate how their locally adapted models behave on task-unrelated inputs rather than transmitting the adaptations themselves. 
In this work, we study two \dsl\ instantiations corresponding to two common collaborative workflows. 
In one-shot foundation-model composition, \dsl\ communicates carrier completions, which can be substantially smaller than adapters, and combines them through distillation, avoiding direct weight-space merging where independently learned specializations can interfere. In iterative federated learning (FL), the same mechanism yields a proxy-free protocol for discriminative models, where clients repeatedly transmit carrier logits on task-unrelated inputs (e.g., images, tabular data, or latent embeddings). 
By avoiding update transmission, \dsl\ removes the direct gradient- and update-leakage surface exploited by many existing attacks \citep{DLG,FCleak_label,gradinv, SecurityPrivacyFederatedSurvey2026d}. While this does not constitute a privacy guarantee, communication is mediated through task-unrelated model behavior rather than direct exposure of the local adaptation. 

Our %\textbf{key contributions} 
contributions are empirical and methodological:
\begin{itemize}[noitemsep, topsep=0.pt, parsep=0.5pt, partopsep=0pt, leftmargin=1.8em]
    % \item We formulate \dsl\ as a collaborative learning primitive for knowledge sharing, and instantiate it in %both 
    % one-shot LLM composition and iterative FL. Participants communicate only random carrier--output pairs, and a coordinator trains a student with a modality-appropriate distillation objective.
    \item We introduce \dsl, a collaborative learning primitive for knowledge sharing, and instantiate it in one-shot LLM composition and iterative FL. Participants communicate only random carrier outputs, % rather than model updates,
    and a coordinator trains a shared model with a modality-appropriate distillation objective.

    % \item In LLM composition, \dsl\ preserves three controlled participant preferences %substantially
    % better than LoRA averaging ($94.56\%$ vs.
    % %versus 
    % $87.76\%$) % on average) 
    % and transfers substantially more mathematical competence %($22.0$-point gain over the base model versus $0.6$ points)
    % ($22.0$ vs. $0.6$ points over the base model), while reducing participant upload by $30.6$--$49.0\times$.

    \item In LLM composition, \dsl\ combines independently learned participant specializations into a single shared model: it retains three controlled preferences better than the LoRA-averaging baseline ($94.56\%$ vs. $87.76\%$) and transfers substantially more mathematical competence ($22.0$ vs. $0.6$ points over the base model), while reducing participant upload by $30.6$--$49.0\times$.

    \item In FL classification, \dsl\ exposes a controllable accuracy--communication frontier: it nearly matches\hspace{-1pt} \fedavg\ \hspace{-2pt} on\hspace{-1pt} MNIST\hspace{-1pt} with\hspace{-1pt} $8.9\times$\hspace{-1pt} less\hspace{-1pt} uplink\hspace{-1pt} and\hspace{-1pt} provides lower-communication operating points on CIFAR-10 and Tiny ImageNet. A local KL-curvature analysis and carrier ablations support the interpretation that success depends on the probed model interface\hspace{-1pt} and\hspace{-1pt} carrier\hspace{-1pt} coverage. %, rather than task semantics alone.
    % A local KL-curvature analysis and carrier ablations support the interpretation that the attainable accuracy at a given communication budget depends on the probed model interface\hspace{-1pt} and\hspace{-1pt} carrier\hspace{-1pt} coverage.

\end{itemize}

% More broadly, these results point toward a different view of collaborative learning, in which knowledge is transmitted through observable model behavior rather than by exchanging the parameters in which that knowledge was learned.

% ------------------------------------------------
% ------------------------------------------------
% ------------------------------------------------
\section{Distributed Subliminal Learning}
\label{sec:method}

% ------------------------------------------------
% ------------------------------------------------
\subsection{Blueprint: communication through random carriers}
\label{sec:blueprint}

Let $\mathcal K=\{1,\ldots,K\}$ denote the set of participants.  Participant $k\in\mathcal K$ holds private data $\mathcal D_k$ and adapts a shared model $\theta$ into a local model $\theta_k=\textsc{LocalTrain}(\theta,\mathcal D_k)$. Classical collaboration communicates $\theta_k-\theta$ or an adapter to the coordinator. \dsl\ instead changes the object through which local knowledge is communicated. Participant $k$ draws $M_k$ carriers $x_{k,i}\sim q$, where $q$ is independent of $\mathcal D_k$, evaluates its local model, and sends
\begin{equation}
    \mathcal M_k=\{(x_{k,i},s_{\theta_k}(x_{k,i}))\}_{i=1}^{M_k}.
    \label{eq:dsl_message}
\end{equation}
The signal $s_{\theta_k}(x)$ denotes the model output communicated through the carrier channel; depending on the instantiation, it can be a logit vector, a probability distribution, or a sampled text completion.
%The signal $s_{\theta_k}$ can be a logit vector, a probability distribution, or a sampled text completion. 
Carriers may be transmitted explicitly or reproduced from shared pseudorandom seeds, so the carrier inputs need not be transmitted. The coordinator pools the messages and optimizes
\begin{equation}
    \theta^{+}
    =\arg\min_{\vartheta}
    \sum_{k=1}^{K}\frac{p_k}{M_k}
    \sum_{i=1}^{M_k}
    d\!\left(f_{\vartheta}(x_{k,i}),s_{\theta_k}(x_{k,i})\right),
    \label{eq:dsl_general}
\end{equation}
where $p_k$ are normalized participant weights and $d$ is the appropriate distillation loss. Our experiments primarily use \emph{private-carrier pooling}, in which participants sample carriers independently and the coordinator distills their pooled carrier outputs into the shared model by minimizing Eq.~(2) over the union of their messages. \emph{Shared-carrier variants}, in which participants evaluate a common carrier set and their outputs are aggregated at the coordinator, and full pseudocode appear in Appendix~\ref{app:protocol}.

Carriers need not resemble the task distribution. Their role is to probe functional directions in which the local models differ from the common model. Consequently, ``random'' does not mean arbitrary: the carrier distribution must elicit outputs that distinguish the relevant local behaviors. We test this dependence by varying carrier distribution, budget, and the probed model interface in the experiments.

% These two instantiations reflect different practical collaborative regimes. Foundation models are commonly adapted from a shared pretrained model and combined after fine-tuning through model or adapter merging \citep{wortsman2022modelsoups,ilharco2023task,yadav2023ties,hu2022lora}. For smaller discriminative models, collaboration more commonly proceeds iteratively through FL, either training the task model from a common initialization or adapting task-specific components on top of a pretrained backbone \citep{mcmahan_FL,kairouz2021advances}. We therefore study \dsl\ both as a one-shot composition mechanism and as a repeated round-based communication primitive.

% ------------------------------------------------
% ------------------------------------------------
\subsection{LLM composition from random completions}
\label{sec:method_llm}
% We first instantiate \dsl\ in a one-shot LLM composition setting, reflecting the common workflow of independently fine-tuning a shared pretrained model and combining the resulting adaptations post hoc.
We first instantiate \dsl\ for one-shot LLM composition. 
Each participant starts from the same pretrained language 
model and applies a local LoRA adaptation \citep{hu2022lora}, yielding a specialized local model.
The carrier prompts are deliberately unrelated to the local specialization. For example, a participant may receive the carrier prompt \emph{``Continue the generation of this sequence of random numbers: 87, 56, 12, \ldots''}. The participant evaluates this prompt with its locally adapted model and records the resulting carrier completion. %The transmitted object is the resulting prompt--completion pair, not the LoRA adapter, its gradient, or any private training example.
% The transmitted object is the resulting completion, as the prompt can be reconstructed from the agreed pseudorandom procedure, not the LoRA adapter, its gradient, or any private training example.
Only this completion is transmitted to the coordinator. The carrier prompt can instead be reconstructed from the agreed pseudorandom procedure; the LoRA adapter, gradients, and private training examples remain local.

Repeating this procedure over independently sampled carriers gives participant $k$ a synthetic message
\begin{equation}
    \mathcal S_k
    =\{(x_{k,i},y_{k,i})\}_{i=1}^{M_k},
    \qquad
    x_{k,i}\sim q,
    \quad
    y_{k,i}\sim p_{\theta_k}(\cdot\mid x_{k,i}),
    \label{eq:llm_message}
\end{equation}
where $q$ generates task-unrelated prompts and $p_{\theta_k}(\cdot\mid x)$ is the local model's conditional output distribution, from which the communicated signal $s_{\theta_k}(x)$ is sampled. The coordinator pools $\mathcal S=\bigcup_k\mathcal S_k$, initializes a shared model from the common base model, and performs supervised fine-tuning:
\begin{equation}
\mathcal L_{\mathrm{SFT}}(\vartheta)
= -\sum_{k=1}^{K}\frac{p_k}{M_k}
\sum_{i=1}^{M_k}\sum_{j=1}^{|y_{k,i}|}
\log p_{\vartheta}\!\left(y_{k,i,j}\mid x_{k,i},y_{k,i,<j}\right).
\label{eq:llm_sft}
\end{equation}
The completions need not explicitly reveal the participant's specialization or contain examples from its task. Under a shared initialization, local adaptation can change the conditional output distribution even on neutral probes; a collection of such completions therefore acts as a function-space signature of the local specialization. Distillation combines these signatures in the shared model, while communication scales with the number and length of completions rather than the adapter size. This is the operational difference from task arithmetic or adapter averaging: the coordinator merges behavior, not parameters.

% ------------------------------------------------
% ------------------------------------------------
\subsection{Federated learning from random logits}
\label{sec:method_fsl}
We next instantiate \dsl\ as an iterative FL protocol for discriminative models, which we call \emph{Federated Subliminal Learning} (\fsl). At round $t$, each client locally adapts the current global model $\theta_t$ into $\theta_{t,k}$, samples random carriers, and uploads the corresponding logit vectors
$z_{t,k}(x)=s_{\theta_{t,k}}(x)$.
Starting from $\theta_t$, the server minimizes a temperature-scaled KL loss over the pooled client messages,
\begin{equation}
\mathcal L_t(\vartheta)
=\sum_{k,i}\frac{p_k}{M_k}\tau^2
\kl\!\left(
\sigma(z_{t,k}(x_{k,i})/\tau)
\,\middle\|\,
\sigma(f_{\vartheta}(x_{k,i})/\tau)
\right).
\label{eq:fsl_loss}
\end{equation}
The carriers are task-agnostic random inputs or random intermediate representations; no public examples or learned generator are used.

% \textbf{Local-linear interpretation.} $\;$
% Suppose the local updates lie in a subspace $\mathcal U$ and the model is linearized around the shared initialization. App.~\ref{app:theory} proves that both shared-carrier averaging and population private-carrier pooling recover the weighted client update whenever the carrier-induced Jacobian Gram matrix is positive definite on $\mathcal U$. A finite-carrier perturbation bound shows convergence as client-specific empirical Gram matrices approach their common expectation. The result is local---it does not claim exact equivalence for arbitrary nonlinear updates---but it predicts two observed requirements: a common base model and sufficient carrier coverage of the adapted interface.

\textbf{Local interpretation.} $\;$
Appendix~\ref{app:theory} shows that, near the current global model, distilling clients' carrier logits can recover the weighted average client update along the parameter directions identified by the carrier outputs. This is characterized by a Fisher-weighted carrier Gram matrix: for population private-carrier pooling, the quadratic Taylor model recovers the weighted update up to its null space, while for shared carriers with averaged logits the equivalence is exact within the linearized logit model. The finite-carrier result applies to the empirical quadratic model, so the analysis provides a local identifiability interpretation rather than a convergence guarantee for the nonlinear SGD %-trained 
protocol.

% ------------------------------------------------
% ------------------------------------------------
% ------------------------------------------------
\section{Experiments}
\label{sec:experiments}
We evaluate \dsl\ across the two collaborative workflows introduced above: one-shot LLM composition %of independently adapted LLMs
through random completions and iterative FL through random-logit exchange. Appendix~\ref{app:experimental_details} provides implementation details, while Appendix~\ref{app:carrier_full} reports the complete federated ablations.

% ------------------------------------------------
% ------------------------------------------------
\subsection{LLM specialization without adapter exchange}
\label{sec:exp_llm}

\textbf{Common protocol and baselines.} $\;$
The controlled preference setting uses \texttt{Qwen2.5-7B-Instruct} and the GSM8K setting uses \texttt{Qwen2.5-1.5B-Instruct}; within each setting, all participants and the coordinator share the same base model. Participants independently fine-tune LoRA adapters and communicate only task-unrelated random-number completions; the coordinator fine-tunes a fresh copy of the base model on their pooled messages. We compare against equal-weight arithmetic averaging of the LoRA adapters, denoted \textsc{LoRA Avg.} \citep{prabhakarLoRASoupsMerging2025}. Upload is measured as UTF-8 completion bytes for \dsl\ and serialized adapter size for \textsc{LoRA Avg.}

% Each participant independently trains a rank-8
% LoRA adapter with $\alpha=8$ on the attention and MLP projection modules
% (\texttt{q\_proj}, \texttt{k\_proj}, \texttt{v\_proj}, \texttt{o\_proj},
% \texttt{gate\_proj}, \texttt{up\_proj}, and \texttt{down\_proj}), using
% Unsloth SFT with a linear learning-rate schedule, 5 warmup steps, and gradient clipping
% at 1.0 (too much for main). In the
% preference experiment, clients are trained on 5,000 synthetic preference examples
% for cat, pasta, and maple respectively, then each generates 40,000 random-number
% continuation completions with temperature 1.0, and a 2048-token
% generation cap. (check later, but if i remember correctly math: 4000 training samples, 12000 carriers)  The server initializes the student from the same base model and trains it on the shuffled union of the three raw completion corpora, using 12 epochs, learning rate $10^{-4}$, batch size 16, and gradient accumulation 4 (same for both).
% We compare against direct arithmetic
% averaging of the three LoRA adapter tensors, denoted \textsc{LoRA Avg.} \cite{prabhakarLoRASoupsMerging2025}.Upload is
% reported as the sum of UTF-8 bytes in the generated completion strings; the LoRA
% baseline is the serialized adapter-weight file size.

% ------------------------------------------------
\subsubsection{Controlled preference composition}
\label{sec:exp_preferences}

\textbf{Setup.} $\;$
We first isolate the communication mechanism in a controlled synthetic setting inspired by subliminal learning \citep{cloud2025subliminal}. Three participants are independently adapted to state that their favorite animal, food, or tree is respectively \emph{cat}, \emph{pasta}, or \emph{maple}, and each communicates 40,000 random-number completions. The goal is to produce one model retaining all three independently injected preferences.  
We report preference retention as the increase over the base model in the fraction of completions containing the target preference word, averaged over 100 completions for each of 30 held-out prompts per category. Table~\ref{tab:pref_composition} summarizes the results.

\textbf{\dsl\ combines independently injected traits with substantially less interference.} $\;$
\dsl\ reaches $94.56\%$ average retention, only $2.98$ points below the local model, whereas \textsc{LoRA Avg.} loses $9.78$ points. %The methods are comparable for cat, where \dsl\ trails by $0.70$ points, but \dsl\ exceeds LoRA averaging by $9.73$ points for pasta and $11.37$ for maple. 
The gain is largest for pasta ($+9.73$ points) and maple ($+11.37$ points). The communicated completions occupy only $1.58$--$2.52$~MiB per participant, corresponding to $30.6$--$48.8\times$ less upload than the adapters. 
This controlled result establishes that neutral completions can carry multiple local behaviors without requiring those behaviors to appear explicitly in the carrier prompts.

\begin{table}[H]
\centering
\caption{Controlled preference composition with Qwen2.5-7B-Instruct. Higher retention is better. \dsl\ upload is measured per participant; reduction is relative to the LoRA payload.}
\label{tab:pref_composition}
\small
\setlength{\tabcolsep}{3.2pt}
\begin{tabular}{lrrrrr}
\toprule
Preference & Local & \textsc{LoRA Avg.} & \dsl & Upload (MiB) & Reduction \\
\midrule
Cat   & 98.23 & 95.60 & 94.90 & 2.52 & $30.6\times$ \\
Pasta & 96.07 & 85.27 & 95.00 & 1.85 & $41.6\times$ \\
Maple & 98.33 & 82.40 & 93.77 & 1.58 & $48.8\times$ \\
\midrule
\textbf{Average} & 97.54 & 87.76 & \textbf{94.56} & -- & -- \\
\bottomrule
\end{tabular}
\end{table}

% ------------------------------------------------
\subsubsection{GSM8K composition}
\label{sec:exp_math}

\textbf{Setup.} $\;$
We next consider task competence learned from local task data rather than injected preferences. Ten participants each fine-tune a LoRA adapter on 4,000 examples sampled from the GSM8K training set and communicate 12,000 task-unrelated random-number completions. We evaluate the %resulting 
models on 500 held-out GSM8K examples using exact-answer accuracy (Table \ref{tab:math_composition}). %The aggregation protocol is unchanged: participants communicate only random-number completions, and the coordinator supervised-fine-tunes one student on their union.

\begin{table}[H]
\centering
\caption{GSM8K composition with Qwen2.5-1.5B-Instruct. $\Delta$ reports the accuracy gain over the common base model; upload is measured per participant.}
\label{tab:math_composition}
\small
\begin{tabular}{lrrrr}
\toprule
Model & Acc. ($\%$) & $\Delta$ Base (pp) & Upload (MiB) & Reduction \\
\midrule
Base model & 10.2 & -- & -- & -- \\
Local models (avg.) & 58.9 & +48.7 & -- & -- \\
\midrule
\textsc{LoRA Avg.} & 10.8 & +0.6 & 35.27 & $1.0\times$ \\
\textbf{\dsl} & \textbf{32.2} & \textbf{+22.0} & \textbf{0.72} & $\mathbf{49.0\times}$ \\
\bottomrule
\end{tabular}
\end{table}

\textbf{\dsl\ transfers task competence where parameter averaging largely fails.} $\;$
The local models reach $58.9\%$ accuracy, confirming that substantial mathematical competence is learned locally. \textsc{LoRA Avg.}, however, reaches only $10.8\%$, just $0.6$ points above the base model. In contrast, \dsl\ reaches $32.2\%$, improving over the base by $22.0$ points and over \textsc{LoRA Avg.} by $21.4$ points. This transfer requires only $0.72$~MiB per participant instead of $35.27$~MiB of adapter parameters, roughly a $49\times$ reduction in upload. This shows that the random-carrier channel can communicate competence learned from local task supervision in addition to injected behavioral traits.

% ------------------------------------------------
% ------------------------------------------------
\subsection{Federated classification}
\label{sec:exp_fsl}

\textbf{Setup and protocols.} $\;$
We evaluate \fsl\ on MNIST with an MLP, CIFAR-10 with ResNet9, and Tiny ImageNet with a frozen ViT-small backbone and trainable downstream head, using 20 clients and three seeds. The main comparison in Table \ref{tab:fsl_overall} uses private-carrier pooling: clients upload only logits on pseudo-random carriers that the server can regenerate. 
We use $M=2{,}048$, $4{,}096$, and $1{,}024$ carriers per client for MNIST, CIFAR-10, and Tiny ImageNet, respectively, with budgets chosen from the carrier sweeps in Appendix~\ref{app:carrier_full} to balance accuracy and communication for each model interface. Appendix~\ref{app:experimental_details} provides full protocol details, while Appendix~\ref{app:carrier_full} additionally compares private, shared, and weighted shared carriers; uniform, blurred, and Fourier distributions; and multiple budgets. %Total uplink is cumulative serialized client upload until early stopping, so communication is interpreted jointly with accuracy.
% -
% OLD We evaluate MNIST with a two-hidden-layer ReLU MLP, CIFAR-10 with ResNet9, and Tiny ImageNet with a frozen ViT-small backbone and trainable downstream head. All experiments use 20 clients, full client participation, at most 200 rounds, and validation-based early stopping. Results are mean $\pm$ standard deviation over three seeds. Unless otherwise stated, client data are IID random shards with 20\% of each shard held out for validation. 
% The main comparison uses private-carrier pooling: clients upload only logits on pseudo-random carriers that the server can regenerate. We use $M=2048$ carriers per client for MNIST, $M=4096$ for CIFAR-10, and $M=1024$ for Tiny ImageNet. MNIST samples independent uniform pixel carriers, CIFAR-10 uses spatially blurred random images, and Tiny ImageNet samples independent uniform vectors at the trainable representation interface. 
 
We compare against parameter-sharing baselines \fedavg\ and \textsc{scaffold} \citep{mcmahan_FL,karimireddySCAFFOLD2020}, and distillation-based alternatives. \textsc{FedMD} exchanges logits on %shared 
public data \citep{liFedMD2019a}; \textsc{FedDF} uploads client models for distillation on unlabeled or generated proxy data \citep{linEnsembleDistillationRobust2020}; and \textsc{FedGen} and \textsc{dfrd} combine model communication with %server-side 
generators  \citep{zhuDataFreeKnowledgeDistillation2021,luoDFRDDataFreeRobustness2023}. In contrast, \fsl\ communicates only outputs on task-agnostic random carriers, requiring neither model/update uploads, \hspace{-1pt}public or proxy data, nor a %learned
generator.

% \textbf{\fsl\ establishes a dominant Pareto frontier, trading marginal accuracy drops for order-of-magnitude uplink savings.} $\;$
\textbf{\fsl\ traces a favorable accuracy--communication frontier, trading accuracy for substantial uplink savings.} $\;$
On MNIST, \fsl\ reaches $96.83\%$, only $0.51$ points below \fedavg, with $8.9\times$ less uplink. % It also approaches \textsc{dfrd} ($97.24\%$) with roughly one eighth of its communication.
On CIFAR-10, it reduces uplink by $136.3\times$ at a $9.94$-point accuracy cost; increasing the carrier budget $M$ and refining the carrier design raises accuracy to $63.02\pm1.15\%$ while retaining a $7.08\times$ reduction relative to \fedavg\ (Table~\ref{tab:cifar_carrier_full}).
On Tiny ImageNet, \fsl\ reaches $63.77\%$, $4.35$ points below \fedavg, with $7.5\times$ less uplink. % It exceeds \textsc{FedGen} by $2.14$ points while using about $1.9\times$ less uplink.
These results support \fsl\ as a tunable communication primitive rather than a fixed drop-in replacement for parameter aggregation.

\begin{wrapfigure}{r}{0.545\textwidth}
    \centering
    \vspace{-10pt}
    \includegraphics[width=\linewidth]{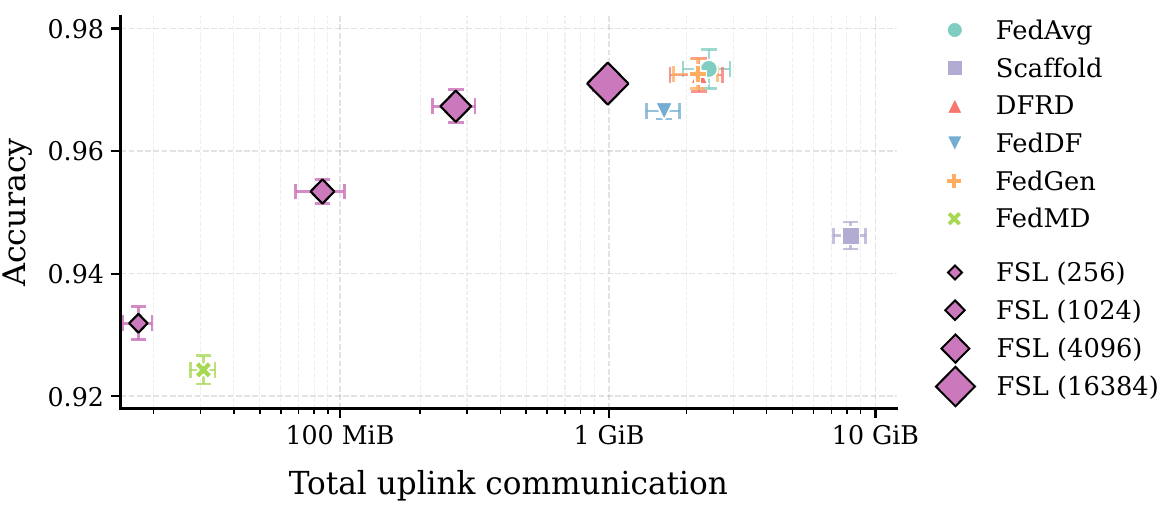}
    \vspace{-12pt}
    \caption{MNIST test accuracy versus total uplink. Increasing the carrier budget moves \fsl\ along an explicit accuracy--communication frontier under the same accounting as the baselines.}
    \label{fig:mnist_fsl_pareto}
    \vspace{-8pt}
\end{wrapfigure}
\textbf{Carrier budget and design control the information exposed to \fsl.} $\;$
Figure~\ref{fig:mnist_fsl_pareto} shows that larger carrier sets progressively close the MNIST accuracy gap while increasing uplink. Broader ablations (Appendix~\ref{app:carrier_full}, Secs.~\ref{app:cifar_carriers}--\ref{app:tinyimg_carriers}) show that carrier suitability depends on the probed model interface rather than task semantics alone: independent uniform carriers are strongest for dense or representation-space interfaces, whereas blurred or Fourier-structured carriers are more effective when probing convolutional networks in image space. Thus, the carriers need not carry task semantics, but they must excite the trainable interface in directions that distinguish the local models. Appendix~\ref{app:carrier_full} also reports reduced-data checks, output-entropy diagnostics, and additional protocol ablations.

\begin{table}[H]
    \centering
    \caption{Test accuracy and total uplink across datasets, reported as mean $\pm$ standard deviation. % over three seeds. 
    %Test accuracy and total uplink across datasets. Results are mean $\pm$ standard deviation. Higher accuracy ($\uparrow$) and lower communication ($\downarrow$) are better. \fsl\ uses private carriers: uniform input-space carriers for MNIST--MLP, blurred input-space carriers for CIFAR-10--ResNet9, and uniform representation-space carriers for Tiny ImageNet.}
    }
    \label{tab:fsl_overall}
    \scriptsize
    \setlength{\tabcolsep}{2.8pt}
    \resizebox{\textwidth}{!}{
    \begin{tabular}{lcc|cc|cc}
        \toprule
        & \multicolumn{2}{c|}{MNIST}
        & \multicolumn{2}{c|}{CIFAR-10}
        & \multicolumn{2}{c}{Tiny ImageNet} \\
        \cmidrule(lr){2-3}\cmidrule(lr){4-5}\cmidrule(lr){6-7}
        Method
        & Acc. (\%) $\uparrow$ & Uplink (GiB) $\downarrow$
        & Acc. (\%) $\uparrow$ & Uplink (GiB) $\downarrow$
        & Acc. (\%) $\uparrow$ & Uplink (GiB) $\downarrow$ \\
        \midrule
        \fedavg
        & $97.34 \pm 0.32$ & $2.37 \pm 0.473$
        & $66.25 \pm 0.67$ & $14.99 \pm 1.47$
        & $68.12 \pm 0.05$ & $0.726 \pm 0.048$ \\
        \textsc{scaffold}
        & $94.62 \pm 0.22$ & $8.07 \pm 1.12$
        & $56.00 \pm 2.60$ & $85.72 \pm 18.40$
        & $65.35 \pm 0.73$ & $19.22 \pm 2.66$ \\
        \textsc{dfrd}
        & $97.24 \pm 0.27$ & $2.18 \pm 0.485$
        & $66.46 \pm 0.61$ & $11.94 \pm 0.283$
        & $68.13 \pm 0.18$ & $0.685 \pm 0.189$ \\
        \textsc{FedDF}
        & $96.65 \pm 0.13$ & $1.61 \pm 0.230$
        & $66.11 \pm 0.60$ & $10.14 \pm 1.58$
        & $66.92 \pm 0.26$ & $0.489 \pm 0.169$ \\
        \textsc{FedGen}
        & $97.26 \pm 0.24$ & $2.15 \pm 0.408$
        & $66.39 \pm 1.16$ & $9.32 \pm 3.22$
        & $61.63 \pm 0.52$ & $0.182 \pm 0.024$ \\
        \textsc{FedMD}
        & $92.43 \pm 0.23$ & $0.030 \pm 0.003$
        & $40.25 \pm 0.69$ & $0.003 \pm 0.001$
        & $45.68 \pm 0.07$ & $0.153 \pm 0.081$ \\
        \midrule
        \textbf{\fsl}
        & $96.83 \pm 0.27$ & $0.266 \pm 0.048$
        & $56.31 \pm 0.54$ & $0.110 \pm 0.031$
        & $63.77 \pm 0.15$ & $0.097 \pm 0.023$ \\
        \bottomrule
    \end{tabular}}
\end{table}

% \begin{figure}[t]
%     \centering
%     \includegraphics[width=0.66\textwidth]{images/mnist_accuracy_uplink_fsl_finegrain_v2.pdf}
%     \caption{MNIST test accuracy versus total uplink. Increasing the carrier budget moves \fsl\ along an explicit accuracy--communication frontier under the same accounting as the baselines.}
%     \label{fig:mnist_fsl_pareto}
% \end{figure}

% --------------------------------------------
% --------------------------------------------
% --------------------------------------------
\section{Related Work}
\label{sec:related}

\textbf{Model merging and subliminal transfer.} $\;$
Model merging combines independently adapted models in parameter space. Model soups and Fisher merging average model weights \citep{wortsman2022modelsoups,matenaMergingModelsFisher2022}, while task arithmetic operates on fine-tuning deltas and methods such as 
\textsc{ties} and \textsc{dare} mitigate interference or redundancy  % between them
\citep{ilharco2023task,yadav2023ties,yuLanguageModelsSuperMario2024}. Recent work also studies composition directly at the LoRA-adapter level \citep{prabhakarLoRASoupsMerging2025}. \dsl\ instead composes local specializations through generated behavior, avoiding parameter-space merging and making communication depend on the carrier payload rather than adapter size. Its mechanism is inspired by subliminal learning \citep{cloud2025subliminal}, which showed that behavioral traits can transfer through semantically unrelated generated data. \dsl\ reverses this original safety framing: hidden transfer through unrelated data becomes an explicit channel for collaborative knowledge sharing.

\textbf{Federated optimization and communication efficiency.} $\;$
\fedavg\ and its variants aggregate client parameter updates, with methods such as \textsc{FedProx} and \textsc{scaffold} addressing statistical heterogeneity and client drift \citep{mcmahan_FL,li2020federated,karimireddySCAFFOLD2020}. Communication-efficient FL further reduces this cost through structured or sketched updates, update sharding and compression, quantization, and less frequent aggregation \citep{konecnyFederatedLearningStrategies2016,reisizadehFedPAQ2020, fenoglioERIS2026d}. These approaches reduce the size or frequency of parameter exchange but remain update-centric. \dsl\ instead changes the communicated object itself, replacing model updates with outputs on task-unrelated carriers.

\textbf{Federated and data-free distillation.} $\;$
Knowledge distillation transfers model behavior through output distributions rather than parameters \citep{hinton2015distilling}. In FL, \textsc{FedMD} exchanges predictions on shared public data \citep{liFedMD2019a}, while \textsc{FedDF} distills uploaded client models using unlabeled or generated proxy data \citep{linEnsembleDistillationRobust2020}. \textsc{FedGen} and \textsc{dfrd} instead learn server-side generators for data-free distillation while retaining model communication \citep{zhuDataFreeKnowledgeDistillation2021,luoDFRDDataFreeRobustness2023}. More broadly, zero-shot and data-free distillation synthesize surrogate inputs from model predictions or internal statistics \citep{nayakZeroShotKnowledge2019,yinDreamingDistillDataFree2020}. In contrast, \fsl\ combines output-only communication with a fixed, task-agnostic carrier distribution, requiring neither model/update uploads, task-related public or proxy data, nor a learned or optimized generator.

\section{Discussion}
\label{sec:discussion}

\textbf{Limitations.} $\;$ 
Our experiments and analysis assume related models sharing a common initialization; extending carrier-based communication to independently pretrained or heterogeneous models remains open. The local analysis characterizes recovery only on the identifiable update subspace and does not establish global equivalence for arbitrary nonlinear adaptations. 
The choice of carrier distribution and budget is problem dependent, and \dsl\ adds coordinator-side distillation.
Avoiding direct gradient %or parameter 
transmission removes a known leakage channel but provides no privacy guarantee, as carrier outputs and distillation may still reveal local information \citep{jagielskiStudentsParrotTheir2023}. Finally, our evaluation covers one LLM family and mostly IID, full-participation federations. Broader merging baselines \citep{prabhakarLoRASoupsMerging2025,stoicaModelMergingSVD2024}, %heterogeneous clients
non-IID data \citep{fenoglioFLUX2025d, jhunjhunwalaFedvarp2022}, adversarial participants \citep{tangLogitsPoisoningAttack2024, fenoglioFBPs2024i} and decentralized coordination remain to be studied \citep{kalraDecentralizedFederatedLearning2023a}.

\textbf{Benefits and extensibility.} $\;$
\dsl\ exposes a direct accuracy--communication knob: the communicated payload is controlled by the carrier budget and output size rather than the number of locally adapted parameters, making it particularly attractive when coordinator compute is abundant but participant uplink is constrained.
The same primitive supports different carrier interfaces and output modalities and is not tied to a particular collaboration pipeline or local adaptation method, making it applicable in principle to one-shot composition as well as iterative or distributed learning settings.
Moreover, function-space communication need not exclude parameter-space methods: when a merged or averaged model is available, it could initialize subsequent 
carrier-based distillation, enabling hybrid forms of collaboration.

% \dsl\ exposes a direct accuracy--communication knob: the communicated payload is controlled by the carrier budget and output size rather than the number of locally adapted parameters. More generally, \dsl\ defines a communication layer rather than a specific learning algorithm: carrier-output exchange can in principle be embedded in one-shot, iterative, or distributed learning pipelines, with different carrier interfaces and output modalities. This separation also leaves local adaptation largely independent of the communication mechanism. Moreover, function-space communication need not exclude parameter-space methods: when a merged or averaged model is available, it could initialize subsequent carrier-based distillation, enabling hybrid forms% of collaboration.

% ------------------------------------------------
% ------------------------------------------------
% ------------------------------------------------
\section{Conclusion}
\label{sec:conclusion}
Collaborative learning typically exchanges model updates and reconciles local knowledge in parameter space. 
We introduced \emph{Distributed Subliminal Learning} (\dsl), a collaborative learning primitive that replaces parameter exchange with function-space communication through task-unrelated random carriers and distills the resulting carrier outputs into a shared model.
% \dsl\ instead communicates behavior on task-unrelated random carriers and distills the resulting outputs into a shared model. 
In LLM composition, \dsl\ preserves controlled preferences better than LoRA averaging and transfers substantially more mathematical competence, while reducing upload by $30.6$--$49.0\times$; in federated classification, it yields a tunable accuracy--communication frontier, nearly matching \fedavg\ on MNIST while substantially reducing uplink across datasets. Together, these results suggest that model parameters need not be the default currency of collaboration. Useful knowledge can instead be shared through behavior on suitably chosen, task-unrelated carriers.

\begin{ack}
This research was funded by the Swiss National Science Foundation, the European Union's Horizon Europe programme, and the Slovenian Research and Innovation Agency through the projects SmartCHANGE (No.\ 101080965) and XAI-PAC (No.\ Z00P2\_216405).
\end{ack}

\clearpage
% in references.bib

\bibliographystyle{unsrtnat} % bibliography style
\bibliography{references} % references.bib

% ------------------------------------------------
% ------------------------------------------------
% ------------------------------------------------
% ------------------------------------------------
% ------------------------------------------------
% ------------------------------------------------
\newpage
\appendix

\section{Protocol and Analysis}
\label{app:protocol}

\subsection{Generic protocol}
\label{app:generic_protocol}

\begin{algorithm}[H]
\caption{Distributed Subliminal Learning with private-carrier pooling}
\label{alg:dsl}
\KwIn{Initial model $\theta_0$, private datasets $\{\mathcal D_k\}_{k=1}^K$, carrier distribution $q$, carrier budgets $\{M_k\}_{k=1}^K$, rounds $T$, participant weights $p_k$}
\For{$t=0,\ldots,T-1$}{
    Select participants $\mathcal S_t$ and broadcast $\theta_t$\;
    \ForEach{$k\in\mathcal S_t$ \textbf{in parallel}}{
        $\theta_{t,k}\leftarrow\textsc{LocalTrain}(\theta_t,\mathcal D_k)$\;
        Sample $\mathcal X_{t,k}=\{x_{t,k,i}\}_{i=1}^{M_k}\overset{\mathrm{i.i.d.}}{\sim}q$\;
        Form $\mathcal M_{t,k}=\{(x_{t,k,i},s_{\theta_{t,k}}(x_{t,k,i}))\}_{i=1}^{M_k}$\;
        Transmit $\mathcal M_{t,k}$; do not transmit $\theta_{t,k}$ or $\theta_{t,k}-\theta_t$\;
    }
    Pool messages and set
    $\theta_{t+1}\leftarrow
    \textsc{Distill}(\theta_t,\cup_{k\in\mathcal S_t}\mathcal M_{t,k},\{p_k\})$\;
}
\KwOut{Shared model $\theta_T$}
\end{algorithm}

For one-shot LLM composition, $T=1$: participants adapt the same pretrained model, produce random completions, and the coordinator supervised-fine-tunes a shared model on their pooled messages. For iterative \fsl, each round alternates client-side local training and server-side carrier distillation.

\subsection{Carrier aggregation variants}
\label{app:aggregation_variants}

\paragraph{Private-carrier pooling.}
Each participant samples its own carrier set and the coordinator minimizes the participant-weighted objective in Eq.~\eqref{eq:dsl_general}. Equal carrier counts and uniform example weighting implement uniform participant weights. Data-size weighting can be implemented explicitly through $p_k$ or by sampling carrier counts proportionally.

\paragraph{Shared-carrier averaging.}
All participants evaluate the same carrier set $\{x_i\}_{i=1}^{M}$. For vector-valued outputs, the coordinator forms
$\bar s(x_i)=\sum_k p_k s_{\theta_k}(x_i)$
and distills against the averaged target. In the FL setting, this is the closest function-space analogue of \fedavg.

\paragraph{Communication object.}
Carriers may be transmitted explicitly or regenerated from an agreed pseudorandom procedure. In the LLM experiments, carrier prompts are reproducible and upload counts only the UTF-8 bytes of the communicated completions. In \fsl, pseudo-random carriers are regenerated by the server and uplink counts the transmitted logits.

\subsection{Local interpretation of the KL objective}
\label{app:theory}

Let $z_{\theta}(x)\in\R^c$ denote the logit vector and fix the common model
$\theta_0$ broadcast at the start of a round. Write
$\delta_k=\theta_k-\theta_0$, $\bar\delta=\sum_k p_k\delta_k$, and
$J(x)=\nabla_{\theta}z_{\theta_0}(x)$. For temperature $\tau$, define
\begin{equation}
    \pi_0(x)=\sigma\!\left(z_{\theta_0}(x)/\tau\right),
    \qquad
    C_0(x)=\operatorname{Diag}(\pi_0(x))-\pi_0(x)\pi_0(x)^\top,
\end{equation}
and the Fisher-weighted carrier Gram matrix
\begin{equation}
    H_q
    =\E_{x\sim q}\!\left[J(x)^\top C_0(x)J(x)\right].
    \label{eq:restricted_gram}
\end{equation}
The matrix $C_0(x)$ is the Hessian of the log-sum-exp function at the
base natural parameter. In particular, it removes directions that add the same scalar
to every logit and are therefore invisible to softmax.

\begin{theorem}[Quadratic local equivalence]
\label{thm:linear_equivalence}
Assume $p_k\geq0$, $\sum_kp_k=1$, and $H_q$ is finite. Suppose the logits are twice
continuously differentiable near $\theta_0$ and their local Taylor remainder
is dominated by a $q$-integrable envelope. For perturbations satisfying
$\max(\|\delta\|_2,\max_k\|\delta_k\|_2)\leq\rho$, the population
private-carrier objective in Eq.~\eqref{eq:fsl_loss} has the expansion
\begin{equation}
    \mathcal L_{\mathrm{private}}(\theta_0+\delta)
    =
    \frac12\sum_k p_k
    (\delta-\delta_k)^\top H_q(\delta-\delta_k)
    +o(\rho^2).
    \label{eq:private_linear_loss}
\end{equation}
The minimizers of the quadratic term are exactly
$\bar\delta+\ker(H_q)$. Equivalently, its identifiable component is the
weighted client update:
\begin{equation}
    \operatorname{Proj}_{\operatorname{range}(H_q)}(\delta^\star)
    =
    \operatorname{Proj}_{\operatorname{range}(H_q)}(\bar\delta).
\end{equation}
If $H_q\succ0$ on the considered trainable parameterization, the quadratic
minimizer is unique and equals $\bar\delta$.

For shared carriers with weighted logit target
$\bar z(x)=\sum_kp_k z_{\theta_k}(x)$, the same minimizer set is exact under
the linearized logit model
$z_{\theta_0+\delta}(x)=z_{\theta_0}(x)+J(x)\delta$.
\end{theorem}

\begin{proof}
At fixed $x$, let $A(\eta)=\log\sum_j e^{\eta_j}$. The KL divergence
between two softmax distributions is the Bregman divergence generated by
$A$, whose Hessian at the base natural parameter
$z_{\theta_0}(x)/\tau$ is $C_0(x)$. Therefore,
\[
\tau^2\kl\!\left(
\sigma(z_{\theta_k}(x)/\tau)
\,\middle\|\,
\sigma(z_{\theta_0+\delta}(x)/\tau)
\right)
=
\frac12(\delta-\delta_k)^\top
J(x)^\top C_0(x)J(x)(\delta-\delta_k)
+o(\rho^2).
\]
Taking expectations and summing over clients gives
Eq.~\eqref{eq:private_linear_loss}. Completing the square yields
\[
\sum_kp_k(\delta-\delta_k)^\top H_q(\delta-\delta_k)
=
(\delta-\bar\delta)^\top H_q(\delta-\bar\delta)
+
\sum_kp_k(\delta_k-\bar\delta)^\top
H_q(\delta_k-\bar\delta),
\]
so the quadratic minimizers are $\bar\delta+\ker(H_q)$.

For shared carriers, the linearized weighted logit target is
$z_{\theta_0}(x)+J(x)\bar\delta$. The KL loss is zero at
$\delta=\bar\delta$, and it is also zero only when
$J(x)(\delta-\bar\delta)$ is a common-logit shift for $q$-almost every $x$.
This condition is equivalent to $\delta-\bar\delta\in\ker(H_q)$.
\end{proof}

The private-carrier statement is deliberately local: the exact population
minimizer of a nonlinear KL mixture need not equal the parameter average away
from the common initialization. The theorem identifies the second-order
mechanism relevant near that initialization; it is not an exact equivalence
claim for the implemented finite optimization.

\subsection{Finite private carriers in the quadratic model}
\label{app:finite_carriers}

Let $U\in\R^{P\times r}$ have orthonormal columns spanning
$\operatorname{range}(H_q)$, let
$G_q=U^\top H_qU\succ0$, and write
$a_k=U^\top\delta_k$. Because the per-carrier curvature matrices are
positive semidefinite, every direction in $\ker(H_q)$ is also annihilated
$q$-almost surely by $J(x)^\top C_0(x)J(x)$; the empirical quadratic loss
therefore depends only on the coordinates in $U$. For client $k$, define
\[
\widehat G_k=\frac{1}{M_k}\sum_{i=1}^{M_k}
U^\top J(x_{k,i})^\top C_0(x_{k,i})J(x_{k,i})U,
\qquad
E_k=\widehat G_k-G_q.
\]
The identifiable component of the empirical quadratic minimizer of
$\frac12\sum_kp_k(a-a_k)^\top\widehat G_k(a-a_k)$ satisfies
\begin{equation}
\widehat a
=\left(\sum_kp_k\widehat G_k\right)^{-1}
\left(\sum_kp_k\widehat G_ka_k\right),
\label{eq:finite_solution}
\end{equation}
whenever the matrix is invertible.

\begin{proposition}[Deterministic finite-carrier perturbation bound]
\label{prop:finite_bound}
Let $\mu=\lambda_{\min}(G_q)>0$. If
$\|E_k\|_2\leq\varepsilon<\mu$ for every client, then
\begin{equation}
\|\widehat a-\bar a\|_2
\leq
\frac{\varepsilon}{\mu-\varepsilon}
\sum_kp_k\|a_k-\bar a\|_2,
\qquad
\bar a=\sum_kp_ka_k.
\label{eq:finite_bound}
\end{equation}
\end{proposition}

\begin{proof}
Set $A=\sum_kp_k\widehat G_k$. Since
$\|A-G_q\|_2\leq\varepsilon$, Weyl's inequality gives
$\lambda_{\min}(A)\geq\mu-\varepsilon$ and
$\|A^{-1}\|_2\leq(\mu-\varepsilon)^{-1}$. Moreover,
\[
\widehat a-\bar a
=A^{-1}\sum_kp_kE_k(a_k-\bar a).
\]
The triangle inequality yields Eq.~\eqref{eq:finite_bound}.
\end{proof}

If the carriers are i.i.d. from the same $q$ and
$\|U^\top J(x)^\top C_0(x)J(x)U\|_2\leq B$ almost surely, matrix
concentration and a union bound give
$\max_k\|E_k\|_2=
O\!\left(B\sqrt{\log(rK/\eta)/M_{\min}}\right)$ with probability at
least $1-\eta$, up to the usual lower-order $M_{\min}^{-1}$ term.
Thus the proposition controls carrier-sampling error in the empirical
quadratic model. It does not control the KL Taylor remainder, server
optimization error, or nonlinear round-to-round dynamics. The bound worsens
when carrier coverage is poorly conditioned (small $\mu$) or client updates
are widely dispersed; the present experiments vary carrier distribution and budget, qualitatively probing the former. %  but do not include a controlled client-heterogeneity sweep.

\section{Experimental Details}
\label{app:experimental_details}

% ------------------------------------------------
\subsection{LLM composition details}
\label{app:llm_details}

\paragraph{Common configuration.}
The controlled preference experiment uses \texttt{unsloth/Qwen2.5-7B} \texttt{-Instruct} as its common base model, while the GSM8K experiment uses \texttt{unsloth/Qwen2.5-1.5B-Instruct}. Participants independently fine-tune rank-8 LoRA adapters with $\alpha=8$ on the attention and MLP projection modules
(\texttt{q\_proj}, \texttt{k\_proj}, \texttt{v\_proj}, \texttt{o\_proj},
\texttt{gate\_proj}, \texttt{up\_proj}, and \texttt{down\_proj}).
Local adaptation uses Unsloth SFT with a linear learning-rate schedule, 5 warmup steps, and gradient clipping at 1.0.
For \dsl\ composition, the coordinator initializes a fresh copy of the common base model and supervised-fine-tunes it on the shuffled union of participant completions for 12 epochs, with learning rate $10^{-4}$, batch size 16, and gradient accumulation 4.
Random-number carrier prompts are generated from randomized templates containing 3--8 integers sampled from 100 to 999 and ask the model to generate up to 10 additional numbers with at most three digits. Carrier completions are sampled with temperature 1.0 and a 2048-token generation cap; other decoding settings use the \texttt{vLLM 0.10.0} defaults.
The \textsc{LoRA Avg.} baseline is obtained by equal-weight arithmetic averaging of the participant LoRA tensors.

\paragraph{Controlled preference composition.}
Three participants are respectively adapted on 5,000 synthetic examples expressing the \emph{cat}, \emph{pasta}, and \emph{maple} preferences. 
Each then generates 40,000 random-number continuations. For evaluation, we use 30 held-out prompts per category and sample 100 completions for each prompt. Preference retention is the mean target-string rate after subtracting the corresponding base-model rate.

\paragraph{GSM8K composition.}
Ten participants independently sample 4,000 examples from the GSM8K training split, without replacement within each client. %; subsets may overlap across clients. 
Each participant fine-tunes its LoRA adapter for one epoch with learning rate $10^{-4}$, batch size 16, gradient accumulation 4, and maximum sequence length 1,024. Each then generates 12,000 task-unrelated random-number completions for \dsl\ composition. Models are evaluated on 500 held-out GSM8K test examples using greedy decoding. We report exact-answer accuracy after extracting and normalizing the final predicted and reference answers.

% ------------------------------------------------
\subsection{Federated protocol summary}
\label{app:federated_details}
Figure~\ref{fig:fsl_overview} illustrates the round-based \fsl\ protocol used throughout these experiments. All federated experiments use 20 clients with full participation, at most 200 communication rounds, and validation-based early stopping. Results are reported as mean $\pm$ standard deviation over three seeds. Unless otherwise stated, client data are IID random shards, with 20\% of each local shard reserved for validation. Both local client training and server distillation use Adam. Clients train for 5 local epochs with learning rate $5\times10^{-4}$, while the server distills for up to 5 epochs with learning rate $3\times10^{-4}$. Batch sizes are 1,024 for MNIST, 256 for CIFAR-10, and 64 for Tiny ImageNet, and the KL objective uses $\tau=1$. Round-level early stopping uses a minimum improvement of $10^{-4}$ with patience 6 for MNIST and CIFAR-10 and patience 10 for Tiny ImageNet. The MNIST MLP has two 512-unit ReLU hidden layers. Tiny ImageNet uses a pretrained \texttt{timm} \texttt{vit\_small\_patch16\_224} frozen feature extractor with a trainable two-layer 512-unit MLP head.

The main comparison in Table~\ref{tab:fsl_overall} uses private-carrier pooling. MNIST--MLP uses $M=2{,}048$ independent uniform pixel carriers per client; CIFAR-10--ResNet9 uses $M=4{,}096$ spatially blurred random-image carriers; and Tiny ImageNet uses $M=1{,}024$ independent uniform representation-space carriers at the trainable interface of a frozen ViT-small backbone. Clients communicate only the resulting logits, while pseudo-random carriers are reproducible at the server.

The controlled MNIST and CIFAR-10 input-space experiments use three
\emph{ghost outputs} by default: additional logit dimensions with no
associated task class or direct supervision, whose carrier logits are
nevertheless communicated to increase the dimensionality of the behavioral
signal. Tiny ImageNet representation-space experiments use none. Server distillation uses the complete communicated output vector, including all ghost logits; no auxiliary ghost heads are used.
We study their effect separately in Appendix~\ref{app:ghost_ablation}.

\begin{figure}[H]
    \centering
    \includegraphics[width=0.99\linewidth]{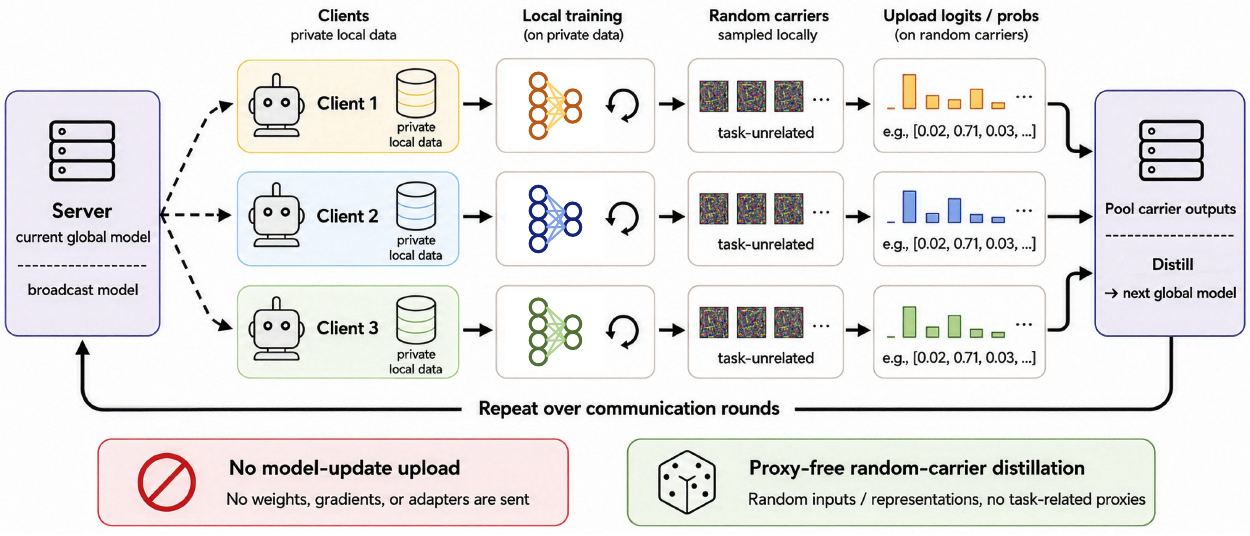}
    \caption{Federated Subliminal Learning. In each communication round, clients train locally, probe their models with task-unrelated carriers, and upload only the resulting carrier outputs. The server pools these outputs and distills the next global model, without receiving model updates or task-related proxy data.}
    \label{fig:fsl_overview}
\end{figure}

% ------------------------------------------------
\subsection{Communication accounting}
\label{app:communication_accounting}

For parameter-sharing federated baselines, total uplink is the cumulative number of serialized client-update bytes transmitted before early stopping. For \fsl, it is the cumulative carrier-output upload over the same period. Because unsuccessful configurations may terminate early, communication is interpreted jointly with accuracy.

For LLM composition, \dsl\ upload is the UTF-8 size of each participant's transmitted completion corpus; the random carrier prompts are reproducible and are therefore not included in the transmitted payload. \textsc{LoRA Avg.} upload is the serialized LoRA adapter size, and reduction factors are computed relative to this payload.

\section{Additional Federated Random-Carrier Experiments}
\label{app:carrier_full}
% =====================================================================

This appendix reports the complete carrier-distribution, aggregation-protocol,
carrier-budget, and communication results underlying
Sec.~\ref{sec:exp_fsl}.
The controlled input-space experiments use 20 clients, a maximum of 200
communication rounds, early stopping, and three ghost outputs
(Appendix~\ref{app:ghost_ablation}) unless otherwise specified.
The additional Tiny ImageNet experiments instead use private-carrier pooling
with no ghost outputs and inject carriers directly at the representation
interface of a frozen ViT-small backbone.

\subsection{Carrier protocols and interpretation of the carrier budget}
\label{app:carrier_protocols}

For completeness, we distinguish the three protocols used in the ablations.

\paragraph{Private-carrier pooling.}
Each participating client $k$ independently samples
$\mathcal{X}_{t,k} \sim q$ and communicates the resulting carrier--logit
pairs. The server concatenates the client-specific sets and distills on their
union.
Consequently, a nominal per-client carrier budget $M$ can expose the server
to substantially more than $M$ distinct probes across the federation.
For example, with 20 clients, $M=2{,}048$ corresponds to up to approximately
$41$k independently generated carrier instances when all clients
participate.

\paragraph{Shared-carrier averaging.}
All participating clients evaluate the same carrier set
$\mathcal{X}_t$.
For every carrier $x\in\mathcal{X}_t$, the server forms
\begin{equation}
    \bar{z}_t(x)
    =
    \frac{1}{|\mathcal{S}_t|}
    \sum_{k\in\mathcal{S}_t} z_{t,k}(x),
\end{equation}
and distills the global model from the resulting common carrier--target set.

\paragraph{Weighted shared-carrier averaging.}
The weighted variant instead uses
\begin{equation}
    \bar{z}_t(x)
    =
    \sum_{k\in\mathcal{S}_t}
    w_k z_{t,k}(x),
    \qquad
    \sum_{k\in\mathcal{S}_t} w_k = 1,
\end{equation}
where $w_k$ denotes the normalized client aggregation weight.

Unlike private pooling, both shared protocols force all client information
through the same $M$-element carrier basis.
This difference is especially important at small $M$.

\paragraph{Carrier distributions.}
Uniform carriers sample each input or representation coordinate independently over the value range observed in the training data. Blurred carriers are obtained by applying Gaussian blur ($\sigma=1.2$) to Gaussian noise, standardizing it per sample/channel, rescaling it using the training-data mean and standard deviation, and clamping it to the observed range. Fourier carriers instead sample complex Gaussian Fourier coefficients with amplitude scaled as $1/f^\alpha$ with $\alpha=1$, remove the DC component, inverse-transform and standardize the result, and apply the same mean--standard-deviation rescaling and clamping.

\paragraph{Communication under early stopping.}
The reported uplink corresponds to the total synthetic-message communication
accumulated before termination.
Because training uses early stopping, low communication is not automatically
a positive result: configurations that fail to learn can stop after very few
rounds and therefore transmit substantially less data.
Accordingly, the tables below should be interpreted jointly in terms of
accuracy and communication rather than communication alone.

% =====================================================================
\subsection{CIFAR-10 with ResNet9}
\label{app:cifar_carriers}
% =====================================================================

Table~\ref{tab:cifar_carrier_full} reports the complete CIFAR-10 sweep.
At the smallest budget, blurred carriers are substantially more
sample-efficient than the alternatives.
Increasing the shared carrier basis substantially improves both blurred and
Fourier carriers, whereas uniform pixel noise remains ineffective even at the
largest tested budget.

\begin{table*}[t]
    \centering
    \caption{
    Complete CIFAR-10--ResNet9 carrier sweep.
    Accuracy and total synthetic uplink are reported as mean $\pm$ standard
    deviation.
    ``Shared-w.'' denotes weighted shared-carrier averaging.
    }
    \label{tab:cifar_carrier_full}
    \small
    \begin{tabular}{rllcc}
        \toprule
        $M$
        & Protocol
        & Carrier
        & Accuracy (\%) $\uparrow$
        & Total uplink $\downarrow$ \\
        \midrule

        4,096 & Private
        & Uniform
        & $41.60 \pm 15.42$
        & $144.90 \pm 112.31$ MiB \\
        4,096 & Private
        & Blurred
        & $55.44 \pm 0.76$
        & $96.60 \pm 13.06$ MiB \\
        4,096 & Private
        & Fourier
        & $39.38 \pm 12.63$
        & $120.52 \pm 145.48$ MiB \\

        \midrule

        4,096 & Shared
        & Uniform
        & $14.58 \pm 0.52$
        & $27.08 \pm 10.22$ MiB \\
        4,096 & Shared
        & Blurred
        & $51.17 \pm 8.64$
        & $193.65 \pm 83.49$ MiB \\
        4,096 & Shared
        & Fourier
        & $35.87 \pm 2.14$
        & $31.15 \pm 15.38$ MiB \\

        \midrule

        4,096 & Shared-w.
        & Uniform
        & $11.67 \pm 0.94$
        & $5.42 \pm 2.35$ MiB \\
        4,096 & Shared-w.
        & Blurred
        & $56.27 \pm 2.09$
        & $219.38 \pm 53.12$ MiB \\
        4,096 & Shared-w.
        & Fourier
        & $44.98 \pm 10.10$
        & $125.94 \pm 105.86$ MiB \\

        \midrule

        32,768 & Shared
        & Uniform
        & $21.88 \pm 0.18$
        & $86.67 \pm 18.76$ MiB \\
        32,768 & Shared
        & Blurred
        & $60.06 \pm 1.75$
        & $920.83 \pm 276.41$ MiB \\
        32,768 & Shared
        & Fourier
        & $55.76 \pm 4.28$
        & $736.67 \pm 411.52$ MiB \\

        \midrule

        32,768 & Shared-w.
        & Uniform
        & $19.41 \pm 1.02$
        & $65.00 \pm 0.00$ MiB \\
        32,768 & Shared-w.
        & Blurred
        & $62.27 \pm 0.57$
        & $1.40$ GiB $\pm 245.37$ MiB \\
        32,768 & Shared-w.
        & Fourier
        & $61.16 \pm 0.46$
        & $1.35$ GiB $\pm 473.58$ MiB \\

        \midrule

        65,536 & Shared-w.
        & Uniform
        & $18.03 \pm 0.77$
        & $130.00 \pm 0.00$ MiB \\
        65,536 & Shared-w.
        & Blurred
        & $63.02 \pm 1.15$
        & $2.12$ GiB $\pm 603.95$ MiB \\
        65,536 & Shared-w.
        & Fourier
        & $62.06 \pm 1.15$
        & $2.20$ GiB $\pm 695.02$ MiB \\

        \bottomrule
    \end{tabular}
\end{table*}

The shared-carrier results make the interaction between carrier geometry and
coverage especially clear.
With blurred carriers, increasing the weighted shared basis from $4{,}096$
to $32{,}768$ and $65{,}536$ raises accuracy from $56.27\%$ to
$62.27\%$ and $63.02\%$.
Fourier carriers exhibit an even larger relative improvement,
from $44.98\%$ to $61.16\%$ and $62.06\%$.
Uniform carriers remain below $20\%$ at the two largest budgets.
Thus, increasing $M$ can compensate for an initially under-sampled but useful
carrier family, as with Fourier, but does not necessarily compensate for a
carrier distribution that poorly probes the function exposed to the carrier
channel.

% =====================================================================
\subsection{MNIST with an MLP}
\label{app:mnist_mlp_carriers}
% =====================================================================

Table~\ref{tab:mnist_mlp_carrier_full} reports the corresponding MNIST MLP
experiments.
Unlike the convolutional models probed directly in image space, the MLP
consistently favors uniform pixelwise carriers.
Private pooling is already effective with $2{,}048$ carriers per client,
while the shared protocols progressively approach the same performance as
the common carrier basis grows.

\begin{table*}[t]
    \centering
    \caption{
    Complete MNIST--MLP carrier sweep.
    Accuracy and total synthetic uplink are reported as mean $\pm$ standard
    deviation.
    ``Shared-w.'' denotes weighted shared-carrier averaging.
    }
    \label{tab:mnist_mlp_carrier_full}
    \small
    \begin{tabular}{rllcc}
        \toprule
        $M$
        & Protocol
        & Carrier
        & Accuracy (\%) $\uparrow$
        & Total uplink $\downarrow$\\
        \midrule

        2,048 & Private
        & Uniform
        & $96.28 \pm 1.21$
        & $175.10 \pm 58.39$ MiB \\
        2,048 & Private
        & Blurred
        & $94.98 \pm 0.06$
        & $153.02 \pm 8.46$ MiB \\
        2,048 & Private
        & Fourier
        & $93.78 \pm 0.24$
        & $109.69 \pm 16.63$ MiB \\

        \midrule

        2,048 & Shared
        & Uniform
        & $92.70 \pm 0.34$
        & $192.97 \pm 41.23$ MiB \\
        2,048 & Shared
        & Blurred
        & $88.34 \pm 0.81$
        & $77.19 \pm 19.59$ MiB \\
        2,048 & Shared
        & Fourier
        & $82.94 \pm 4.16$
        & $52.81 \pm 35.42$ MiB \\

        \midrule

        2,048 & Shared-w.
        & Uniform
        & $92.39 \pm 0.15$
        & $146.25 \pm 16.12$ MiB \\
        2,048 & Shared-w.
        & Blurred
        & $87.77 \pm 1.41$
        & $65.68 \pm 18.21$ MiB \\
        2,048 & Shared-w.
        & Fourier
        & $84.96 \pm 1.23$
        & $53.49 \pm 13.06$ MiB \\

        \midrule

        16,384 & Shared
        & Uniform
        & $95.58 \pm 0.18$
        & $1.39$ GiB $\pm 170.69$ MiB \\
        16,384 & Shared
        & Blurred
        & $93.79 \pm 0.15$
        & $1.20$ GiB $\pm 124.11$ MiB \\
        16,384 & Shared
        & Fourier
        & $93.42 \pm 0.13$
        & $1.16$ GiB $\pm 197.98$ MiB \\

        \midrule

        16,384 & Shared-w.
        & Uniform
        & $95.69 \pm 0.14$
        & $1.42$ GiB $\pm 119.78$ MiB \\
        16,384 & Shared-w.
        & Blurred
        & $92.58 \pm 0.57$
        & $677.08 \pm 315.24$ MiB \\
        16,384 & Shared-w.
        & Fourier
        & $92.68 \pm 0.60$
        & $812.50 \pm 197.02$ MiB \\

        \midrule

        32,768 & Shared-w.
        & Uniform
        & $96.72 \pm 0.21$
        & $2.67$ GiB $\pm 424.99$ MiB \\
        32,768 & Shared-w.
        & Blurred
        & $93.86 \pm 0.20$
        & $1.87$ GiB $\pm 148.93$ MiB \\
        32,768 & Shared-w.
        & Fourier
        & $93.48 \pm 0.82$
        & $1.93$ GiB $\pm 1.14$ GiB \\

        \bottomrule
    \end{tabular}
\end{table*}

The MLP results differ systematically from the image-space convolutional
experiments.
Uniform carriers are strongest for every tested shared-weighted budget and
improve from $92.39\%$ at $2{,}048$ carriers to $95.69\%$ at $16{,}384$ and
$96.72\%$ at $32{,}768$.
Blurred and Fourier carriers also improve with $M$, but remain consistently
below the uniform configuration.
This pattern is compatible with independent pixelwise probes providing broad
directions for a dense input layer, while spatial smoothing removes variation
that the MLP could otherwise exploit.

% =====================================================================
\subsection{MNIST with LeNet5}
\label{app:mnist_lenet_carriers}
% =====================================================================

To disentangle dataset effects from architectural effects, we repeat the
MNIST experiment with LeNet5.
Table~\ref{tab:mnist_lenet_carrier_full} shows that replacing the MLP by a
convolutional model qualitatively changes the preferred carrier family.

\begin{table*}[t]
    \centering
    \caption{
    Complete MNIST--LeNet5 carrier sweep.
    Accuracy and total synthetic uplink are reported as mean $\pm$ standard
    deviation.
    ``Shared-w.'' denotes weighted shared-carrier averaging.
    }
    \label{tab:mnist_lenet_carrier_full}
    \small
    \begin{tabular}{rllcc}
        \toprule
        $M$
        & Protocol
        & Carrier
        & Accuracy (\%) $\uparrow$
        & Total uplink $\downarrow$\\
        \midrule

        1,024 & Private
        & Uniform
        & $75.58 \pm 33.80$
        & $116.12 \pm 104.89$ MiB \\
        1,024 & Private
        & Blurred
        & $95.09 \pm 0.87$
        & $196.35 \pm 28.04$ MiB \\
        1,024 & Private
        & Fourier
        & $94.98 \pm 0.51$
        & $179.09 \pm 6.12$ MiB \\

        \midrule

        1,024 & Shared
        & Uniform
        & $13.95 \pm 4.15$
        & $3.05 \pm 1.02$ MiB \\
        1,024 & Shared
        & Blurred
        & $13.25 \pm 1.89$
        & $3.39$ MiB $\pm 600.44$ KiB \\
        1,024 & Shared
        & Fourier
        & $36.20 \pm 40.92$
        & $89.04 \pm 148.94$ MiB \\

        \midrule

        1,024 & Shared-w.
        & Uniform
        & $14.05 \pm 4.13$
        & $3.05 \pm 1.02$ MiB \\
        1,024 & Shared-w.
        & Blurred
        & $13.25 \pm 1.92$
        & $3.39$ MiB $\pm 600.44$ KiB \\
        1,024 & Shared-w.
        & Fourier
        & $37.41 \pm 43.04$
        & $96.48 \pm 161.84$ MiB \\

        \midrule

        4,096 & Shared
        & Uniform
        & $13.50 \pm 3.19$
        & $10.82 \pm 2.34$ MiB \\
        4,096 & Shared
        & Blurred
        & $92.02 \pm 1.31$
        & $796.83 \pm 162.86$ MiB \\
        4,096 & Shared
        & Fourier
        & $92.56 \pm 1.34$
        & $1.09$ GiB $\pm 355.42$ MiB \\

        \midrule

        4,096 & Shared-w.
        & Uniform
        & $39.06 \pm 47.45$
        & $346.67 \pm 582.86$ MiB \\
        4,096 & Shared-w.
        & Blurred
        & $91.14 \pm 2.20$
        & $758.33 \pm 225.54$ MiB \\
        4,096 & Shared-w.
        & Fourier
        & $88.78 \pm 2.22$
        & $617.50 \pm 251.65$ MiB \\

        \midrule

        8,192 & Shared-w.
        & Uniform
        & $52.75 \pm 40.18$
        & $796.25 \pm 999.77$ MiB \\
        8,192 & Shared-w.
        & Blurred
        & $93.71 \pm 0.36$
        & $1.70$ GiB $\pm 223.15$ MiB \\
        8,192 & Shared-w.
        & Fourier
        & $92.73 \pm 1.96$
        & $1.40$ GiB $\pm 676.09$ MiB \\
        \bottomrule
    \end{tabular}
\end{table*}

The private-carrier result is particularly informative because the dataset is
unchanged relative to the MLP experiment.
With only $1{,}024$ private carriers per client, blurred and Fourier carriers
both reach approximately $95\%$ accuracy, while uniform carriers are much
less reliable.
The shared-carrier experiments additionally show that the common carrier basis
must be sufficiently large: blurred and Fourier configurations fail at
$1{,}024$ shared carriers but recover above $88\%$--$92\%$ at $4{,}096$ and
above $92\%$ at $8{,}192$.
Uniform carriers remain substantially less reliable even as $M$ grows.

Taken together with the MNIST MLP results, this experiment provides evidence
that carrier suitability depends on the architecture exposed to the carrier
channel rather than being determined by the dataset alone.

% =====================================================================
\subsection{Representation-Space Carriers on Tiny ImageNet}
\label{app:tinyimg_carriers}
% =====================================================================

The preceding experiments inject carriers in the original image space.
We additionally consider a complementary setting on Tiny ImageNet in which
the ViT-small representation extractor is frozen and random carriers are
injected directly at the embedding interface of the trainable downstream
layers.
This experiment uses private-carrier pooling, 20 clients, a maximum of
200 communication rounds, early stopping, and no ghost outputs.

This setting provides an additional test of whether carrier suitability is
determined by the original input modality or by the function actually probed
by the carrier channel.
Since the frozen transformer is bypassed during carrier communication, the
random representations directly probe the trainable downstream mapping.

\begin{table*}[t]
    \centering
    \caption{
    Private representation-space carrier sweep on Tiny ImageNet using
    50k training samples and a frozen ViT-small feature extractor.
    Accuracy and total synthetic uplink are reported as mean $\pm$
    standard deviation.
    The $M=1{,}024$ uniform configuration is the Tiny ImageNet FSL
    setting reported in Table~\ref{tab:fsl_overall}.
    }
    \label{tab:tinyimg_carrier_full}
    \small
    \begin{tabular}{rllcc}
        \toprule
        $M$
        & Protocol
        & Carrier
        & Accuracy (\%) $\uparrow$
        & Total uplink $\downarrow$\\
        \midrule

        64
        & Private
        & Uniform
        & $60.54 \pm 0.58$
        & $33.53 \pm 2.03$ MiB \\

        128
        & Private
        & Uniform
        & $61.84 \pm 0.35$
        & $44.27 \pm 17.51$ MiB \\

        256
        & Private
        & Uniform
        & $62.74 \pm 0.32$
        & $74.22 \pm 14.08$ MiB \\

        \midrule

        512
        & Private
        & Blurred
        & $60.12 \pm 0.19$
        & $101.56 \pm 0.00$ MiB \\

        512
        & Private
        & Uniform
        & $63.18 \pm 0.32$
        & $83.33 \pm 16.26$ MiB \\

        \midrule

        1,024
        & Private
        & Blurred
        & $60.46 \pm 0.33$
        & $151.04 \pm 47.74$ MiB \\

        1,024
        & Private
        & Uniform
        & $63.77 \pm 0.15$
        & $98.96 \pm 23.87$ MiB \\

        \midrule

        2,048
        & Private
        & Blurred
        & $61.29 \pm 0.18$
        & $197.92 \pm 36.08$ MiB \\

        2,048
        & Private
        & Uniform
        & $63.13 \pm 0.40$
        & $239.58 \pm 36.08$ MiB \\

        \midrule

        4,096
        & Private
        & Blurred
        & $61.17 \pm 0.19$
        & $312.50 \pm 0.00$ MiB \\

        4,096
        & Private
        & Uniform
        & $63.10 \pm 0.36$
        & $375.00 \pm 125.00$ MiB \\

        \bottomrule
    \end{tabular}
\end{table*}

Uniform representation-space carriers consistently outperform blurred
carriers at every matched carrier budget.
At $M=512$, $1{,}024$, $2{,}048$, and $4{,}096$, the respective uniform
accuracies are $63.18\%$, $63.77\%$, $63.13\%$, and $63.10\%$, compared
with $60.12\%$, $60.46\%$, $61.29\%$, and $61.17\%$ for blurred carriers.

The uniform budget sweep additionally exhibits a clear saturation effect.
Performance improves from $60.54\%$ with only 64 carriers to $62.74\%$
with 256 and $63.77\%$ with 1,024, after which increasing the carrier set
does not improve accuracy and substantially increases communication.
We therefore use the $M=1{,}024$ configuration in the cross-dataset
comparison in Table~\ref{tab:fsl_overall}.

This result complements the raw-image experiments in the main paper.
Although Tiny ImageNet is itself an image dataset, the carrier channel here
does not probe the frozen visual representation extractor.
Instead, the random representations directly probe the trainable downstream
layers.
The preference for independent uniform carriers is therefore consistent with
the broader hypothesis that carrier geometry should match the function and
representation interface through which knowledge is communicated.

\paragraph{Robustness to the available training data.}
We additionally repeat the matched uniform-versus-blurred comparison after
reducing the Tiny ImageNet training pool from 50k to 30k samples.
Table~\ref{tab:tinyimg_30k} reports the resulting accuracies.
We focus on $M\geq512$, for which both carrier distributions were evaluated.

\begin{table}[t]
    \centering
    \caption{
    Tiny ImageNet representation-space carrier comparison with a reduced
    30k-sample training pool.
    Results are mean $\pm$ standard deviation.
    }
    \label{tab:tinyimg_30k}
    \small
    \begin{tabular}{rcc}
        \toprule
        $M$
        & Uniform (\%)
        & Blurred (\%) \\
        \midrule

        512
        & $\mathbf{61.40 \pm 0.45}$
        & $57.45 \pm 0.25$ \\

        1,024
        & $\mathbf{62.04 \pm 0.28}$
        & $58.15 \pm 0.10$ \\

        2,048
        & $\mathbf{61.67 \pm 0.36}$
        & $58.28 \pm 0.20$ \\

        4,096
        & $\mathbf{61.34 \pm 0.36}$
        & $58.69 \pm 0.19$ \\

        \bottomrule
    \end{tabular}
\end{table}

Reducing the training pool lowers absolute accuracy, as expected, but
preserves the carrier ordering at every matched budget.
Uniform carriers outperform blurred carriers by approximately
$2.7$--$4.0$ percentage points across the sweep.
The consistency of this effect under a different amount of training data
provides additional evidence that the observed preference is associated with
the representation-space carrier interface rather than with a particular
training-set size.

% =====================================================================
\subsection{Complete teacher-output entropy diagnostics}
\label{app:carrier_entropy}
% =====================================================================

Table~\ref{tab:entropy_full} reports the complete output-entropy measurements available on CIFAR-10. Besides the private 4,096-carrier configuration, we evaluate larger weighted shared-carrier settings.

For teacher probability vector $p_i$ on carrier $x_i$ and
$\bar p=M^{-1}\sum_{i=1}^M p_i$, we use
\[
H_{\mathrm{query}}=\frac{1}{M\log C}\sum_{i=1}^M H(p_i),
\qquad
H_{\mathrm{marg}}=\frac{H(\bar p)}{\log C},
\]
where $C$ is the number of classes. The first quantity measures average target
softness; the second measures aggregate coverage of the teacher's output
classes.

\begin{table*}[t]
    \centering
    \caption{
    Teacher-output diagnostics on CIFAR-10.
    $H_{\mathrm{query}}$ is the normalized mean per-carrier class entropy and
    $H_{\mathrm{marg}}$ is the normalized entropy of the mean class
    distribution.
    Results are mean $\pm$ standard deviation.
    }
    \label{tab:entropy_full}
    \small
    \begin{tabular}{llr l cc}
        \toprule
        Model
        & Protocol
        & $M$
        & Carrier
        & $H_{\mathrm{query}}$ (\%)
        & $H_{\mathrm{marg}}$ (\%) \\
        \midrule

        ResNet9 & Private & 4,096 & Uniform
        & $31.11 \pm 4.79$
        & $75.95 \pm 1.91$ \\
        ResNet9 & Private & 4,096 & Blurred
        & $43.31 \pm 2.47$
        & $84.33 \pm 2.20$ \\
        ResNet9 & Private & 4,096 & Fourier
        & $38.65 \pm 0.04$
        & $79.68 \pm 4.48$ \\

        \midrule

        ResNet9 & Shared-w. & 32,768 & Uniform
        & $41.42 \pm 4.46$
        & $53.49 \pm 5.32$ \\
        ResNet9 & Shared-w. & 32,768 & Blurred
        & $49.40 \pm 0.99$
        & $91.98 \pm 1.33$ \\
        ResNet9 & Shared-w. & 32,768 & Fourier
        & $42.82 \pm 4.34$
        & $88.05 \pm 1.86$ \\

        \midrule

        MLP & Private & 4,096 & Uniform
        & $75.62 \pm 1.31$
        & $95.42 \pm 0.43$ \\
        MLP & Private & 4,096 & Blurred
        & $54.76 \pm 8.33$
        & $90.70 \pm 1.42$ \\
        MLP & Private & 4,096 & Fourier
        & $48.87 \pm 2.91$
        & $92.89 \pm 0.95$ \\

        \midrule

        MLP & Shared-w. & 16,384 & Uniform
        & $89.61 \pm 0.83$
        & $99.41 \pm 0.05$ \\
        MLP & Shared-w. & 16,384 & Blurred
        & $66.99 \pm 3.77$
        & $97.24 \pm 0.17$ \\
        MLP & Shared-w. & 16,384 & Fourier
        & $61.26 \pm 7.64$
        & $96.97 \pm 0.45$ \\

        \bottomrule
    \end{tabular}
\end{table*}

The larger shared-carrier experiment strengthens the ResNet9 observation.
At $M=32{,}768$, uniform carriers exhibit substantially lower marginal output
entropy ($0.535$) than blurred ($0.920$) or Fourier ($0.881$), matching the
large performance gap among these configurations.
For the MLP, uniform carriers instead achieve the highest marginal entropy in
both the private and shared-weighted settings.

We emphasize that these statistics should be interpreted as diagnostics
rather than sufficient conditions for successful transfer.
In particular, high per-query entropy may simply indicate uncertain teacher
outputs.
Marginal entropy is more directly related to the extent to which the carrier
set elicits outputs across the available class space, but it does not by
itself measure whether those outputs encode all functional directions required
for distillation.

% =====================================================================
\subsection{Ablation on the number of ghost outputs}
\label{app:ghost_ablation}
% =====================================================================

Ghost outputs are additional logit dimensions with no associated task class
and no direct supervision. Although they do not affect the task prediction,
their values on random carriers are included in the communicated output
vector and can therefore provide additional dimensions through which local
model behavior is expressed.

The controlled MNIST and CIFAR-10 input-space experiments use three ghost
outputs by default. Figure~\ref{fig:ghost_ablation} reports the corresponding
sensitivity analysis on MNIST. Increasing the number of ghost outputs
generally improves the attainable accuracy, but also increases the amount of
synthetic information communicated. Since this mechanism is orthogonal to
the carrier-distribution and carrier-interface questions studied in
Sec.~\ref{sec:exp_fsl}, we fix it to three for the corresponding input-space
experiments. The Tiny ImageNet representation-space experiments use no ghost
outputs and are reported separately in Sec.~\ref{app:tinyimg_carriers}.

\begin{figure}[H]
    \centering
    \includegraphics[width=0.55\linewidth]{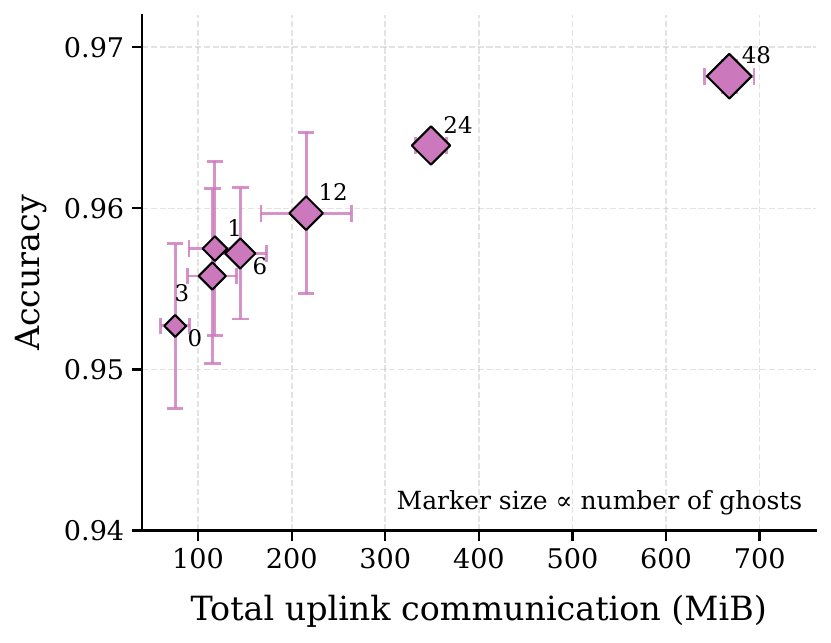}
    \caption{
    Accuracy--communication trade-off as the number of ghost outputs is
    varied on MNIST.
    Marker size is proportional to the number of ghosts.
    Error bars report the standard error.
    }
    \label{fig:ghost_ablation}
\end{figure}

% The NeurIPS checklist is not requested on the workshop CFP page. If the
% organizers later require it, uncomment the following lines.
% \clearpage
% \input{checklist.tex}

\end{document}